\documentclass[12pt]{article}
\usepackage[utf8]{inputenc}
\usepackage[margin=1.375in]{geometry}
\usepackage{parskip}
\usepackage{bibentry}
\usepackage[ruled,vlined]{algorithm2e}
\usepackage{booktabs}
\usepackage{csquotes}
\usepackage{comment}

\usepackage{caption, subcaption}
\DeclareCaptionLabelFormat{no-parens}{\textsf{\textbf{#2}}}
\usepackage[linkcolor=black,colorlinks=true,citecolor=black,filecolor=black, backref=page]{hyperref}
\renewcommand*{\backrefalt}[4]{%
       \ifcase #1 %
         No citations.%
       \or
         p. #2.%
       \else
         pp. #2.%
\fi }

\usepackage{enumitem}
\usepackage{natbib}

\usepackage{thm-restate, amsmath, amssymb, amsfonts, mathtools}
\usepackage[ruled,vlined]{algorithm2e}

\usepackage{amsthm}
\makeatletter
\newtheorem*{rep@theorem}{\rep@title}
\newcommand{\newreptheorem}[2]{%
\newenvironment{rep#1}[1]{%
 \def\rep@title{#2 \ref{##1}}%
 \begin{rep@theorem}}%
 {\end{rep@theorem}}}
\makeatother

\newtheorem{theorem}{Theorem}[subsection]
\newreptheorem{theorem}{Theorem}
\newtheorem{fact}[theorem]{Fact}

\newreptheorem{assumption}{Assumption}
\newtheorem{lemma}[theorem]{Lemma}

\newreptheorem{lemma}{Lemma}

\newreptheorem{corollary}{Corollary}

\DeclareMathOperator*{\argmin}{arg\,min}
\newcommand{\norm}[1]{\left\lVert#1\right\rVert}
\newcommand{\snorm}[1]{\lVert#1\rVert}
\newcommand{\pder}[2]{\frac{\partial#2}{\partial #1}}
\newcommand{\spder}[2]{\partial_#1#2}
\newcommand{\der}[2]{\frac{\mathrm{d}#2}{\mathrm{d} #1}}

\newcommand{\Id}{\mathrm{Id}}
\newcommand{\evalat}[2]{\left. #1 \right\rvert_{#2}}

\newcommand{\Lt}{\mathcal{L}}
\newcommand{\Li}{L^{\mathrm{in}}}
\newcommand{\Lo}{L^{\mathrm{out}}}
\newcommand{\outg}{\nabla_\theta}
\newcommand{\eoutg}{\widehat{\nabla}_\theta}
\newcommand{\phisb}{\phi_\beta^*}
\newcommand{\phisbt}{\phi_{\theta, \beta}^*}
\newcommand{\phisz}{\phi_0^*}
\newcommand{\phiszt}{\phi_{\theta, 0}^*}
\newcommand{\phist}{\phi_{\theta}^*}

\newcommand{\phihb}{\hat{\phi}_\beta}
\newcommand{\phihz}{\hat{\phi}_0}

\title{Beyond backpropagation: bilevel optimization through implicit differentiation and equilibrium propagation}
\author{Nicolas Zucchet\textsuperscript{1}\textsuperscript{,}\thanks{To whom correspondence may be addressed: \texttt{nzucchet@inf.ethz.ch}} , João Sacramento\textsuperscript{2}\\
\\
\textsuperscript{1}Department of Computer Science, ETH Zurich, Switzerland\\
\textsuperscript{2}Institute of Neuroinformatics,\\University of Zurich and ETH Zurich, Switzerland\\
}
\date{}

\begin{document}

\maketitle

\normalsize
\begin{abstract}
    \normalsize
    This paper reviews gradient-based techniques to solve bilevel optimization problems. Bilevel optimization extends the loss minimization framework underlying statistical learning to systems that are implicitly defined through a quantity they minimize. This characterization can be applied to neural networks, optimizers, algorithmic solvers and even physical systems, and allows for greater modeling flexibility compared to the usual explicit definition of such systems. We focus on solving learning problems of this kind through gradient descent, leveraging the toolbox of implicit differentiation and, for the first time applied to this setting, the equilibrium propagation theorem. We present the mathematical foundations behind such methods, introduce the gradient estimation algorithms in detail, and compare the competitive advantages of the different approaches.
\end{abstract}

Recent years have witnessed an explosion of breakthroughs fueled by deep learning in many scientific fields such as computer vision \citep{krizhevsky_imagenet_2012}, natural language processing \citep{brown_language_2020}, game playing \citep{mnih_human-level_2015} and biology \citep{jumper_highly_2021}. There are many lessons to learn from these advances. Among them is the surprising effectiveness of gradient descent: updating the millions, or even billions, of parameters of a deep learning model through greedy gradient-following updates turns out to be extremely powerful and cheap, thanks to the backpropagation of errors algorithm \citep{linnainmaa_taylor_1976, werbos_applications_1982, rumelhart_learning_1986}.

In its standard form, backpropagation provides an efficient way of computing gradients in neural networks, but its applicability is limited to acyclic directed computational graphs whose nodes are explicitly defined. Feedforward neural networks or unfolded-in-time recurrent neural networks are prime examples of such graphs. However, there exists a wide range of computations that are easier to describe through what they \emph{achieve}, rather than by the exact sequence of calculations they perform, and that thus do not fulfill the requirements of backpropagation. This includes outputs of algorithmic solvers which provably minimize some cost function \citep{djolonga_differentiable_2017, wang_satnet_2019, vlastelica_differentiation_2020}, of learning processes that do loss minimization \citep{mackay_practical_1992, bengio_gradient-based_2000} and even of physical systems, such as biological neural networks \citep{hopfield_neurons_1984} or electrical circuits \citep{wyatt_criteria_1989, kendall_training_2020, scellier_deep_2021}, reaching a steady state. Gradient descent based on naive backpropagation cannot improve those computations, as the algorithm is not directly applicable.

In this article, we frame learning such implicitly-defined systems as a bilevel optimization problem and review how to compute the corresponding gradients through implicit differentiation methods. We then present a less explored alternative approach, which relies on the equilibrium propagation theorem, recently discovered by \citet{scellier_equilibrium_2017}.

Our article is organized as follows:
\begin{itemize}
    \item[--] In \textbf{Section 1}, we formalize bilevel optimization and discuss some examples in which it appears in machine learning from a historical perspective. We then focus on hyperparameter optimization and meta-learning to highlight the challenges behind solving bilevel optimization problems.
    \item[--] Behind the tools of interest for this article are the concept of implicit function and the so-called implicit function theorem. In \textbf{Section 2}, we provide the reader insight into why this notion is so fundamental and present the implicit function theorem in detail.
    \item[--] \textbf{Section 3} is the core of the paper: we there introduce two gradient-based approaches to solve bilevel optimization problems. The first class of methods leverages the differentiation formula provided by the implicit function theorem, while the second one builds on another theorem, the equilibrium propagation theorem. We present the mathematical foundations of the two approaches and show how to transform them into efficient gradient estimation algorithms. We theoretically analyze the quality of the gradient those algorithms produce as a function of the different sources of approximation they introduce.
    \item[--] In \textbf{Section 4}, we compare the different algorithms we presented in the last section with each other, exhibiting their qualities and limitations. We then discuss when these methods shine by contrasting them with the following alternatives: backpropagation through the optimization process and black-box optimization strategies.
\end{itemize}
This article has two levels of reading: one for the reader interested in learning new gradient estimation methods and one for the reader who wishes to know the mathematical foundations behind them. For this reason, we mark all theory-oriented sections with the symbol \dag. They can be skipped without hindering the understanding of the rest.


\section{Bilevel optimization in machine learning}

\subsection{Bilevel optimization}

The high-level description of bilevel optimization that we briefly sketched above contains two elements: an inner optimization process which describes what the system does, and an outer loss function that ultimately measures how good the result of this process is. We now make this formulation more precise.

Let us denote by $\phi$ the parameters that are optimized by the inner process to minimize the inner loss function $\Li$. The system we consider has some parameters $\theta$ that we want to learn. We assume that they modify the behavior of the system through $\Li$. The computation performed by the system is then
\begin{equation*}
    \phist \in \argmin_\phi \Li(\phi, \theta).
\end{equation*}
We use the subscript $\theta$ to underline that $\phist$ is an implicit function of $\theta$ (as $\Li$ depends on $\theta$), as it can be any local minimizer of the inner loss $\Li$. Note that we do not make any assumption on how to obtain $\phist$, as we only assume that it minimizes a loss function.

The outer loss $\Lo$ measures the quality of the output $\phist$ of the system and plays the usual role of a loss function in machine learning. We can then frame learning of the parameters $\theta$ as the minimization of the outer loss, which leads to the bilevel optimization problem that we study in this article:
\begin{equation}
    \label{eqn:bilevel_optim}
    \begin{split}
        & \min_\theta \Lo(\phist, \theta)\\
        & \text{s.t.}~ \phist \in \argmin_\phi \Li(\phi, \theta).
    \end{split}
\end{equation}

\subsection{Historical perspective}

Bilevel optimization was originally introduced in the 1930s by von Stackelberg \citep{von_stackelberg_market_1934} in the context of two-players games with a leader and a follower, and later extensively studied in the field of optimization as a way to model optimization problems that contain different objectives \citep{bard_practical_1998}. Closer to the learning formulation of interest to this article is bilevel optimization as introduced for the training of recurrent neural networks in the late 1980s. Instead of describing neural dynamics by their dynamics and then backpropagating through them, the neural activity is assumed to converge to a steady-state. This view led to the introduction of the recurrent backpropagation algorithm \citep{almeida_learning_1990, pineda_generalization_1987}, one of the algorithms we review in Section~\ref{sec:approximation_methods}. Often, converging dynamics are described as minimizing an energy function \citep{hopfield_neurons_1984, cohen_absolute_1983, scellier_equilibrium_2017, whittington_approximation_2017}. This offers stability guarantees and allows connecting to physical systems such as resistive \citep{millar_general_1951, hutchinson_computing_1988, wyatt_criteria_1989, kendall_training_2020} or flow \citep{stern_supervised_2021} networks.

This implicit characterization of entire neural networks, or layers of them, introduced in the early days of deep learning has regained considerable interest recently \citep{amos_optnet_2017, djolonga_differentiable_2017, agrawal_differentiable_2019, gould_deep_2021}. Notably, a class of such implicit networks called deep equilibrium models\footnote{See \citet{kolter_deep_2021} for a tutorial on the topic.} \citep{bai_deep_2019, bai_multiscale_2020} have achieved state-of-the-art performance in many problem domains. These results demonstrate that the performance of large deep feedforward neural networks can be matched by neural networks with far fewer parameters, when the computations they perform are iterated repeatedly until equilibrium. As we will later see in more detail, this results in large memory savings not only during inference but also during learning. Bilevel optimization also appears in many other forms in modern machine learning, going from hyperparameter optimization and meta-learning, to generative adversarial networks \citep{goodfellow_generative_2014, metz_unrolled_2017} and reinforcement learning \citep{pfau_connecting_2016, rajeswaran_game_2020, nikishin_control-oriented_2022}. We zoom in on hyperparameter optimization and meta-learning in the next section as this is one of the problems for which bilevel optimization is mostly used nowadays. We refer the curious reader to Appendix~\ref{app:bo_formulations} for a more extensive presentation of some existing formulations.

\subsection{Hyperparameter optimization and meta-learning}
\label{subsec:hp_ml}

\paragraph{Hyperparameter optimization.}

Let us consider the following problem: we want to find the parameters $\theta$ of a learning algorithm, its hyperparameters, that generate model parameters $\phi$ which generalize well. We measure generalization performance by testing the learned model on held-out data. Furthermore, as is conventionally done, we assume that model parameters are obtained by maximum a posteriori estimation \citep{mackay_practical_1992, foo_efficient_2007} or, alternatively, by regularized empirical risk minimization \citep{bengio_gradient-based_2000, goutte_adaptive_1998}. This leads to the following bilevel optimization problem:
\begin{equation}
    \begin{split}
        & \min_\theta L(\phi_\theta^*, \mathcal{D}^\text{val})\\
        & \mathrm{s.t.}~ \phi_\theta^* \in \argmin_\phi L(\phi, \mathcal{D}^\text{train}) + R(\phi, \theta),
    \end{split}
\end{equation}
where $L$ is the negative log-likelihood that measures the discrepancy between the predictions of a neural network parameterized by $\phi$ and the target outputs on a dataset $\mathcal{D}$, $\mathcal{D}^\text{train}$ is the training set, $\mathcal{D}^\text{val}$ is a held-out dataset and $R(\phi, \theta)$ is the negative log-prior (in the Bayesian view) or a regularizing term on $\phi$ (in the frequentist view). For instance, a very common choice is to take $R(\phi, \theta) = \lambda \snorm{\phi}^2$; in this case, the hyperparameters are $\theta = \{ \lambda \}$. A zoo of different interactions between $\phi$ and $\theta$ can be considered, and we mention a few of them in Appendix~\ref{app:interactions_meta_learning}.
 
When $\theta$ is low-dimensional, black-box optimization methods such as grid or random search \citep{bergstra_random_2012} can search for the best hyperparameters. However, this becomes intractable for high-dimensional hyperparameters. Alternatively, one could backpropagate through the training trajectory, but this does not scale well with the number of updates, as the entire history of parameters must be stored during training and then revisited in reverse-time order. The implicit methods we present in Section~\ref{sec:approximation_methods} do not suffer from these limitations. They can scale to a large number of hyperparameters and long training procedures.

\paragraph{Meta-learning.}
The previous formulation can be extended to meta-learning \citep{thrun_learning_1998, schmidhuber_evolutionary_1987, bengio_learning_1990, finn_model-agnostic_2017, bertinetto_meta-learning_2019} by considering several tasks. The goal is now to learn meta-parameters $\theta$ that yield a learning algorithm that generalizes well on a family of tasks: ideally, the algorithm will achieve low loss on unseen tasks, which are assumed to be similar to those encountered during meta-learning. The corresponding optimization problem is then:
\begin{equation}
    \begin{split}
        & \min_\theta \mathbb{E}_{\tau} \left [ \Lo(\phi_{\tau,\theta}^*, \theta, \mathcal{D}_\tau^\text{val}) \right ]\\
        & \mathrm{s.t.}~ \phi_{\tau, \theta}^* \in \argmin_\phi \Li(\phi, \theta, \mathcal{D}_\tau^\text{train}),
    \end{split}
\end{equation}
where $\Li$ and $\Lo$ are the same kind of loss used for hyperparameter optimization with the difference that the data on which they are defined is now dependent on the task $\tau$. In practice this is solved by stochastic gradient descent on the expected outer loss over the task distribution: one task (or more) is sampled and the gradient corresponding to that task is approximated in the same way it would be for hyperparameter optimization. Thus, black-box optimization methods and backpropagation through training suffer from the same problems we highlighted above.


\section{The implicit function theorem}

Studying implicit functions is about understanding the relationship between two variables $y$ and $x$, when they are linked together through an equation $f(x,y)=0$. The first apparitions of implicit functions can be traced back to \citet{descartes_geometrie_1637} and \citet{newton_methodis_1670} who considered the behavior of some specific curves \citep{krantz_implicit_2003}. \citet{cauchy_turin_1831} laid down the theoretical foundations behind the implicit function theorem and the extended modern multivariate version of the theorem was introduced by Ulysse Dini in lecture notes\footnote{There is however no trace of the implicit function theorem in the 69 original papers Dini published.} supporting his teaching on infinitesimal analysis at the University of Pisa during the academic year 1877-1878 \citep{scarpello_historical_2002}. We end this historical note with a citation from \citet{euler_introductio_1748} (as translated by John D. Blanton) that perfectly captures why implicit functions are relevant in mathematics in general and which particularly relates to the philosophy behind bilevel optimization:
\blockquote{Indeed frequently algebraic functions cannot be expressed explicitly. For example, consider the function $Z$ of $z$ defined by the equation, $z^5 = az^2 Z^3 - bz^4Z^2 + cz^3 Z - 1$. Even if this equation cannot be solved, still it remains true that $Z$ is equal to some expression composed of the variable $z$ and constants, and for this reason $Z$ shall be a function of $z$.}

Implicit functions are inherent to bilevel optimization as the function $\phist$ used in \eqref{eqn:bilevel_optim} satisfies $\partial_\phi \Li(\phist, \theta) = 0$.
Understanding how implicit functions behave is therefore crucial; this is what the implicit function theorem brings. More precisely, it contains two statements: first, it ensures $\theta \mapsto \phist$ exists locally, and second, it yields an analytical formula for the outer gradient $\nabla_\theta$ associated with our problem:
\begin{equation}
    \label{eqn:outg_implicit_differentiation}
    \begin{split}
        \outg^\top & \coloneqq \der{\theta}{} \Lo(\phist, \theta)\\
        & = \pder{\theta}{\Lo}(\phist, \theta) - \pder{\phi}{\Lo}(\phist, \theta) \left ( \pder{\phi^2}{^2\Li}(\phist, \theta) \right )^{-1} \pder{\theta \partial \phi}{^2 \Li}(\phist, \theta). \\
    \end{split}
\end{equation}
We derive this formula in Section~\ref{subsec:outg_derivation}.

The attentive reader will have noticed that we are using two different notations for derivatives in the outer gradient formula. Let us clarify the convention we follow. We use $\partial_x$ to denote partial derivatives with respect to $x$ and $\mathrm{d}_x$ for total derivatives. There is no difference between partial and total derivatives when the function only depends on one variable. In the multivariate case, this is different. We here use the $\partial_x$ notation when the derivative is straightforward to calculate, as for the gradient of a loss function, and $\mathrm{d}_x$ when the function has some hidden dependency on $x$, as it occurs for implicit functions. We consider both partial and total derivatives of scalar functions to be row vectors, so that $\partial_\theta \Lo$ is a row vector of size $|\theta|$, $\spder{\phi}{\Lo}$ a row vector of size $|\phi|$, $\spder{\phi^2}{\Li}$ a squared matrix of size $|\phi|\times |\phi|$ and $\spder{\theta\partial_\phi}{\Li}$ a matrix of size $|\phi| \times |\theta|$.

Using the outer gradient $\outg$ for gradient descent would in principle yield an efficient algorithm to solve our bilevel optimization problem with the nice property that it only requires knowing $\phist$. Unlike backpropagation-through-time, storing the sequence of intermediate parameter values generated by the learning algorithm is no longer needed. However, computing the outer gradient requires computing the Hessian $\spder{\phi}{^2 \Li}(\phist, \theta)$, which is a second-order derivative, and inverting it. Those two operations are costly and often intractable in large-scale machine learning problems. We therefore need to approximate the outer gradient $\outg$ if we want to use it for practical purposes. This is what the methods we present in Section \ref{sec:approximation_methods} do.

The rest of the section is dedicated to explaining in further detail the statements and consequences of the implicit function theorem for our bilevel optimization problem. It can be skipped on a first reading without impairing the understanding of the rest of the article.

The usual formulation of the implicit function theorem \citep{dontchev_implicit_2009} encompasses both the existence statement and the differentiation formula. We present and discuss next the two parts separately for the sake of clarity.

\subsection{Existence of implicit functions \dag}

In the bilevel optimization formulation (\ref{eqn:bilevel_optim}), we used the implicit function $\phist$ without ensuring that it is correctly defined. The first part of the implicit function theorem ensures that such a function exists.

\begin{theorem}[Existence of implicit functions \citep{dontchev_implicit_2009}]
    \label{thm:ift_existence}
    Let $f$ be continuously differentiable and $(\bar{\phi}, \bar{\theta})$ be such that $f(\bar{\phi}, \bar{\theta}) = 0$. If the Jacobian matrix $\partial_\phi f(\bar{\phi}, \bar{\theta})$ is invertible, there exists a unique continuous implicit function $\theta \mapsto \phi^*_\theta$ defined in a neighborhood of $\bar{\theta}$ such that $\phi_{\bar{\theta}}^* = \bar{\phi}$ and which verifies $f(\phi^*_\theta, \theta)=0$ for all $\theta$ in that neighborhood. 
\end{theorem}

Once applied to the constraint $f = \partial_\phi \Li$ that follows from the local minimality constraint in \eqref{eqn:bilevel_optim}, the invertibility condition becomes an invertibility condition on the Hessian $\partial_\phi^2 \Li(\phist, \theta)$ and the implicit function verifies $\partial_\phi \Li(\phist, \theta)=0$ on the neighborhood on which it is defined. Note that if $\bar{\phi}$ is a minimizer of $\Li(\cdot, \bar{\theta})$ then $\phist$ will also be as long as $\Li$ is twice continuously differentiable\footnote{This can be obtained by remarking that 1. the smallest eigenvalue of an invertible Hessian is strictly positive and 2. the smallest eigenvalue of $\partial_\phi^2 \Li(\phist, \theta)$ is a continuous function of $\theta$. This implies that for $\theta$ in the neighborhood of $\bar{\theta}$ considered in Theorem~\ref{thm:ift_existence}, the smallest eigenvalue of $\partial_\phi^2 \Li(\phist, \theta)$ is strictly positive and hence that $\phist$ is a local minimizer of $\Li$ for every $\theta$ in this neighborhood.}. The implicit function theorem is purely local in the sense that several implicit functions can cohabit for a given $\theta$ but in different regions of the $\phi$ space, as shown on Figure~\ref{fig:example_ift}.B for $\theta > 0$. This why we use the notation $\phist \in \argmin_\phi \Li(\phi, \theta)$ in (\ref{eqn:bilevel_optim}): the problem is still well defined even if there exists several local minima for the same $\theta$.

The main assumption of Theorem~\ref{thm:ift_existence} applied to bilevel optimization is the invertibility of the Hessian at $(\bar{\phi}, \bar{\theta})$. Without this assumption, the graph associated with the minimizers can split, as illustrated in the following example. Let $\Li(\phi, \theta) := \phi^4 - \theta \phi^2$ for $\phi$ and $\theta$ real variables. We plot the graph of this function for several $\theta$ values on Figure~\ref{fig:example_ift}.A. The Hessian, here a second-order derivative, is null when $(\phi,\theta)=(0,0)$ (hence not invertible). A branching behavior occurs at this point since there exists a unique minimizer (which is also the only stationary point) of the function at $\phi=0$ when $\theta$ is negative and three otherwise, see Figure~\ref{fig:example_ift}.B. The graph associated with the implicit functions therefore splits into 3 branches at $\theta=0$, making it impossible to properly define an implicit function in this neighborhood.
\begin{figure}[ht] 
    \centering
    \begin{subfigure}{180pt}
        \caption{}
        \includegraphics{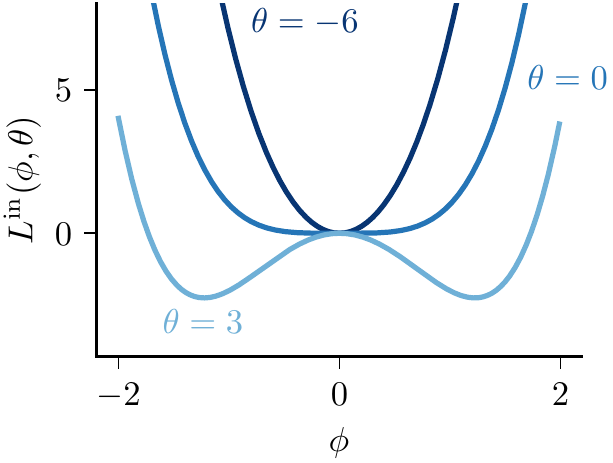}
    \end{subfigure}
    \hspace{20pt}
    \begin{subfigure}{180pt}
        \caption{}
        \includegraphics{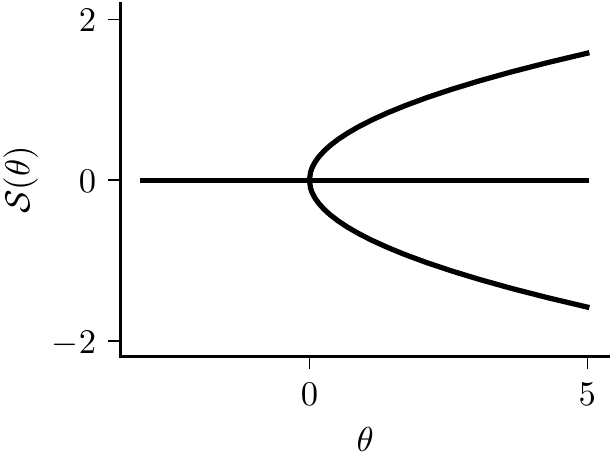}
    \end{subfigure}
    \caption{(A) Visualization of the function $\Li : (\phi, \theta) \mapsto \phi^4 - \theta \, \phi^2$ for several $\theta$ values. (B)  When $\theta$ equals 0, the Hessian at $\phi=0$ is non-invertible which implies that there is no implicit function defined around $\phi =0$, as shown on the graph of the solution mapping $\mathcal{S}(\theta) := \{ \phi \,|\, \partial_\phi \Li(\phi, \theta)=0 \}$ associated to the equilibrium condition $\partial_\phi \Li(\phi, \theta)=0$.}
    \label{fig:example_ift}
\end{figure}

\subsection{Analytical formula for the outer gradient \dag}
\label{subsec:outg_derivation}

Once we know that an implicit function exists, we would like to know how it locally reacts to changes in $\theta$, i.e., if it is differentiable, and if so, what is its derivative. This is what the second part of the implicit function theorem brings.

\begin{theorem}[Differentiating implicit functions \citep{dontchev_implicit_2009}]
    \label{thm:ift_differentiation}
    Under the assumptions of Theorem \ref{thm:ift_existence}, the implicit function $\phist$ defined in Theorem \ref{thm:ift_existence} is differentiable and
    \begin{equation*}
        \der{\theta}{\phi_\theta^*} = -\left(\pder{\phi}{f}(\phist, \theta)\right)^{-1} \pder{\theta}{f}(\phist, \theta).
    \end{equation*}
\end{theorem}
\begin{proof}
    The derivation of the previous formula is relatively straight-forward once we know the differentiable implicit function exists as it only requires differentiating through the constraint using the chain rule: as $f(\phist, \theta) = 0$ for all $\theta$ on which $\phist$ is defined, we have
    \begin{equation*}
        \begin{split}
            0 & = \der{\theta}{}f(\phist,\theta)\\
            & = \pder{\theta}{f}(\phist, \theta) + \pder{\phi}{f}(\phist, \theta)\der{\theta}{\phist},
        \end{split}
    \end{equation*}
which yields the desired formula after rearranging the different terms.
\end{proof}

We can then use this formula to obtain $\mathrm{d}_\theta \phist$ for our bilevel optimization problem
\begin{equation}
    \der{\theta}{\phist} = -\left ( \pder{\phi^2}{^2\Li}(\phist, \theta) \right )^{-1} \pder{\theta \partial \phi}{^2 \Li}(\phist, \theta).
\end{equation}
Together with the chain rule, this is just what we need to obtain an expression for the outer gradient:
\begin{equation}
    \begin{split}
        \nabla_\theta^\top & = \der{\theta}{} \Lo(\phist, \theta)\\
        & = \pder{\theta}{\Lo}(\phist, \theta) + \pder{\phi}{\Lo}(\phist, \theta) \der{\theta}{\phist} \\
        & = \pder{\theta}{\Lo}(\phist, \theta) - \pder{\phi}{\Lo}(\phist, \theta) \left ( \pder{\phi^2}{^2\Li}(\phist, \theta) \right )^{-1} \pder{\theta \partial \phi}{^2 \Li}(\phist, \theta). \\
    \end{split}
\end{equation}
Note that the implicit function theorem and all methods we present here only require the stationarity condition $\partial_\phi \Li(\phist, \theta)=0$ to be satisfied, and not the more restrictive minimality assumption $\phist \in \argmin_\phi \Li(\phi, \theta)$. The methods we introduce in the next section can therefore be easily be extended to solve any optimization problem of the form
\begin{equation}
    \begin{split}
        & \min_\theta \Lo(\phi_\theta^*, \theta)\\
        & \mathrm{s.t.}~ \partial_\phi \Li(\phist, \theta)=0.
    \end{split}
\end{equation}


\section{Approximations of the outer gradient}

\label{sec:approximation_methods}

As we mentioned in the last section, computing the outer gradient using its analytical formula
\begin{equation}
    \label{eqn:outg_implicit_differentiation_v2}
    \outg^\top = \pder{\theta}{\Lo}(\phist, \theta) - \pder{\phi}{\Lo}(\phist, \theta) \left ( \pder{\phi^2}{^2\Li}(\phist, \theta) \right )^{-1} \pder{\theta \partial \phi}{^2 \Li}(\phist, \theta)
\end{equation}
is not feasible in most practical applications of bilevel optimization: we need approximations. Different methods exist to do so. We classify them into two different categories: implicit differentiation methods that approximate the outer gradient by directly using the analytical formula (\ref{eqn:outg_implicit_differentiation_v2}) obtained with the implicit function theorem and equilibrium propagation which leverages an alternative formulation for the outer gradient that we will later present. Note that we have here written the derivative with respect to outer parameters $\theta$ but everything can be transposed to derivatives with respect to inputs, thus allowing us to backpropagate through implicitly defined layers in deep architectures \citep{amos_optnet_2017, gould_deep_2021}.

In the following, we provide intuition behind the different methods, exhibit their fundamental similarities and differences, and compare their theoretical guarantees.

\begin{algorithm}[ht]
    \caption{Side-by-side comparison of implicit differentiation (Section \ref{sec:idb_methods}) and equilibrium propagation (Section \ref{sec:epb_methods}) methods}
    \label{alg:comparison}
    \KwResult{Approximate solution $\theta$ of the bilevel optimization problem (\ref{eqn:bilevel_optim})}
    \For{i = 1, ..., n}{
        Minimize $\Li(\phi,\theta)$ with respect to $\phi$ and note $\hat{\phi}$ the approximate result;\\
        \textbf{Implicit differentiation}\\
        \Indp
        Minimize the quadratic form
        \begin{equation*}
            \pi \mapsto \frac{1}{2}\, \pi\,\pder{\phi^2}{^2\Li}(\hat{\phi}, \theta)\,\pi^\top - \pi \pder{\phi}{\Lo}(\hat{\phi},\theta)^\top
        \end{equation*}
        through e.g. gradient descent or conjugate gradient and note $\hat{\pi}$ the approximate result;\\
        Estimate the outer gradient with
        \begin{equation*}
            \eoutg^\top := \pder{\theta}{\Lo}(\hat{\phi}, \theta) - \hat{\pi} \pder{\theta\partial\phi}{^2 \Li}(\hat{\phi}, \theta);
        \end{equation*}
        \Indm
        \textbf{Equilibrium propagation}\\
        \Indp
        Minimize $\Lt(\phi, \theta, \beta) = \Li(\phi, \theta) + \beta \Lo(\phi, \theta)$ with respect to $\phi$ for some small non-zero $\beta$ value (potentially for more $\beta$ values if needed), starting from $\hat{\phi}$, and note $\hat{\phi}_{\beta}$ the approximate result;\\
        Estimate the outer gradient with
        \begin{equation*}
            \eoutg^\top := \frac{1}{\beta}\left ( \pder{\theta}{\Lt}(\hat{\phi}_{\beta},\theta,\beta) - \pder{\theta}{\Lt}(\hat{\phi},\theta,0) \right )
        \end{equation*}
        or with an estimator that uses more points;\\
        \Indm
        Update $\theta$ using $\widehat{\nabla}_\theta$;
    }
\end{algorithm}

\subsection{Implicit differentiation}

\label{sec:idb_methods}

\paragraph{Gradient computation as minimization of a quadratic form.} Computing and inverting Hessians are costly operations (respectively quadratic and cubic in the size of the differentiated parameter) so the inverse Hessian term in \eqref{eqn:outg_implicit_differentiation_v2} must often be approximated in practice. The first key insight needed for those methods is to iteratively approximate the row vector
\begin{equation}
    \label{eqn:def_pistar}
    \pi^* \coloneqq \pder{\phi}{\Lo}(\phist, \theta) \left ( \pder{\phi^2}{^2\Li}(\phist, \theta) \right )^{-1}
\end{equation}
by minimizing the quadratic form
\begin{equation}
    \label{eqn:quadratic_form_ift}
    \pi \mapsto \frac{1}{2}\, \pi\,\pder{\phi^2}{^2\Li}(\phist, \theta)\,\pi^\top - \pi \pder{\phi}{\Lo}(\phist,\theta)^\top.
\end{equation}
As we are using row vectors, the quantity $\pi \partial_\phi \Lo(\phist, \theta)^\top$ corresponds to a dot product. If $\phist$ is a non-flat local minimizer of $\Li$ then the invertible Hessian condition needed in Theorem \ref{thm:ift_existence} is satisfied and the quadratic form (\ref{eqn:quadratic_form_ift}) is positive definite so it has a unique minimizer, which is $\pi^*$.

\paragraph{Choice of the optimizer.}

Naively minimizing the quadratic form (\ref{eqn:quadratic_form_ift}) does not yet lead to a practical algorithm. Let us take the example of gradient descent. An update would take the form
\begin{equation}
    \label{eqn:rbp}
    \pi \leftarrow \pi - \alpha \left ( \pi\,\pder{\phi^2}{^2\Li}(\phist, \theta) - \pder{\phi}{\Lo}(\phist,\theta)\right )
\end{equation}
with $\alpha$ the learning rate. Evaluating \eqref{eqn:rbp} appears to
require computing the Hessian $\partial_\phi^2 \Li$, and then multiplying it with the vector $\pi$, an operation with quadratic complexity which would render the method impractical. However, there is a way of obtaining the update above without ever having to explicitly calculate the Hessian: a cleverer implementation exploits the fact that all we need is a Hessian-vector product. Remarkably, computing such products has the same complexity as computing gradients \citep{pearlmutter_fast_1994}. Gradient descent, and in fact many other optimization procedures, can therefore be executed efficiently. We call this process the second phase, whereas the first phase consists in computing $\phist$.

Implicit differentiation methods take different forms depending on the choice of the optimizer. When gradient descent is chosen as in (\ref{eqn:rbp}), this leads to recurrent backpropagation, also known as the Almeida-Pineda algorithm \citep{almeida_learning_1990, pineda_generalization_1987}\footnote{The usual  way of deriving recurrent backpropagation is by using iterative updates to find the solution of the linear system $\pi\pder{\phi^2}{^2\Li}(\hat{\phi},\theta) = \pder{\phi}{\Lo}(\hat{\phi}, \theta)$. Although this is equivalent to gradient descent on the quadratic form when applied to bilevel optimization, this view allows considering the more general case in which equilibrium states are not necessarily minimizers of a loss function. Here, we use the quadratic form minimization view as it makes the comparison to other methods easier.}. The very same update of recurrent backpropagation can be obtained from different perspectives. For example, it can be derived starting from the Neumann series formulation of the inverse of a matrix \citep{liao_reviving_2018, lorraine_optimizing_2020}: we have
\begin{equation}
    \label{eqn:neumann_series}
    \begin{split}
        \left(\pder{\phi^2}{^2 \Li}(\phist, \theta)\right)^{-1} &= \alpha \left(\alpha \pder{\phi^2}{^2 \Li}(\phist, \theta)\right)^{-1}\\
        & = \alpha \sum_{i=0}^\infty \left(\Id - \alpha \pder{\phi^2}{^2 \Li}(\phist, \theta)\right)^i
    \end{split}
\end{equation}
whenever the absolute eigenvalues of $(\Id - \alpha \partial_\phi^2 \Li(\phist, \theta))$ are strictly smaller than one (which requires $\alpha$ small enough). We cannot use this formula alone as it still requires computing the Hessian but we can use it to iteratively approximate $\pi^*$ using
\begin{equation}
    \pi \leftarrow \pi \left ( \Id - \alpha \pder{\phi^2}{^2 \Li}(\phist, \theta) \right ) + \alpha \pder{\phi}{\Lo}(\phist,\theta),
\end{equation}
which is exactly the same update as (\ref{eqn:rbp}). This is why we used the same notation for the learning rate and the rescaling parameter even though we introduced those two parameters from different contexts. Alternatively, this update also appears in truncated backpropagation \citep{williams_efficient_1990} when gradient descent on $\Li$ has reached a minimum for several steps \citep{shaban_truncated_2019}. Backpropagating through the last iteration takes exactly the same form as (\ref{eqn:rbp}), but it requires storing the intermediate states in memory as opposed to recurrent backpropagation.

Gradient descent is a very general algorithm. Since we want to minimize a specific class of function, one may ask whether more tailored optimization procedures might be more efficient. This is what the conjugate gradient method provides (we refer to \citet{shewchuk_introduction_1994} for more details on the algorithm), while still only requiring Hessian-vector products.

Note that we can obtain first-order approximations of the outer gradient by limiting the number of steps in the second phase. If we skip the second minimization and approximate the result by $\hat{\pi} = 0$, the corresponding approximate outer gradient will be equal to the direct derivative $\partial_\theta \Lo(\phist, \theta)$. If we perform only one step and take $\alpha=1$, we approximate the Hessian with the identity \citep{luketina_scalable_2016}. The amount of compute attributed to the second phase therefore progressively transforms a first-order approximation toward the true value of the gradient.

\paragraph{Some practical considerations.} In practice, we rarely directly minimize (\ref{eqn:quadratic_form_ift}) as we do not have access to an exact minimizer $\phist$ of the inner loss, but only to an estimate $\hat{\phi}$. Instead, we use the estimated version of the quadratic form (replacing $\phist$ by $\hat{\phi}$) as shown in Algorithm~\ref{alg:comparison}.

In many applications, $\Li$ is an empirical risk, that is the average of some loss evaluated on many different data samples. In this case, it might not be possible to compute Hessian-vector products for all the data at once. To work around this issue we can resort to stochastic updates on the quadratic (taking a random subset of the data for each step), as noted in the lecture notes of \citet{grosse_lecture_2021}.

\paragraph{Robustness to non-optimality \dag.} As mentioned above, the local minimizer $\phist$ is almost always approximated in practice. A natural question to ask is whether the methods introduced above are robust to this approximation. In other words, we may ask how good $\eoutg$ is compared to $\outg$, with
\begin{equation}
    \eoutg^\top = \pder{\theta}{\Lo}(\hat{\phi}, \theta) - \hat{\pi} \pder{\theta\partial\phi}{^2 \Li}(\hat{\phi}, \theta),
\end{equation}
as in Algorithm~\ref{alg:comparison}. In the last equation, $\hat{\pi}$ is obtained by iteratively minimizing the estimated version of the quadratic form (\ref{eqn:quadratic_form_ift}), that consists in replacing $\phist$ by $\hat{\phi}$. Its estimation will therefore be the other source of approximation.

We are doing approximate gradient descent at the outer level, which will result in approximate solutions to the bilevel optimization problem. \citet{daspremont_smooth_2008} and \citet{friedlander_hybrid_2012} have shown that the error made in solving a convex optimization problem with inexact gradients can be linked to the gradient approximation error $\snorm{\eoutg - \outg}$. Motivated by those results we present a theoretical bound on the error made in estimating the outer gradient $\outg $ with $\eoutg$ depending on the quality of $\hat{\phi}$ and $\hat{\pi}$. 

\begin{restatable}{assumption}{ibd_analysis}
    \label{ass:ibd_analysis}
    Suppose that there exists positive real numbers $(\mu, \rho, B, L, M)$ such that:
    \begin{enumerate}
        \item[i.] $\Li$ is twice continuously differentiable and $\Lo$ is continuously differentiable.
        \item[ii.] $\Li$ is $\mu$-strongly convex as a function of $\phi$.
        \item[iii.] The second-order derivatives (Hessian and cross derivatives) of $\Li$ are $\rho$-Lipschitz as functions of $\phi$.
        \item[iv.]  As functions of $\phi$, $\Lo$ is $B$-Lipschitz, $L$-smooth and $\spder{\theta}\Lo$ is $M$-Lipschitz,
    \end{enumerate}
\end{restatable}

\begin{restatable}[Error bound for implicit differentiation methods \citep{pedregosa_hyperparameter_2016}]{theorem}{idb_analysis}
    \label{thm:idb_analysis}
    Let $\phist$ be a minimizer of $\Li$ and $\hat{\phi}$ be its approximated value. Let $\delta \in \, ]0, \frac{\mu}{2\rho}[$ be an upper bound on the corresponding approximation error:
    \begin{equation*}
        \snorm{\phist-\hat{\phi}} \leq \delta.
    \end{equation*}
    Let $\hat{\pi}$ be an approximation of $\pi^*$ computed by one of the implicit differentiation methods and $\delta' > 0$ be an upper bound of its approximation error:
    \begin{equation*}
        \snorm{\hat{\pi} - \spder{\phi}{\Lo}(\hat{\phi}, \theta)\,\spder{\phi}{^2\Li}(\hat{\phi},\theta) ^{-1}} \leq \delta'.
    \end{equation*}
    Then, under Assumption~\ref{ass:ibd_analysis}, there exists a constant $C$ such that
    \begin{equation*}
        \snorm{\nabla_\theta - \widehat{\nabla}_\theta} \leq C(\delta+\delta').
    \end{equation*}
\end{restatable}

The quantities $\delta$ and $\delta'$ measure the error made in the two phases of the algorithms, where the first phase consists in finding a minimum of $\Li$ and the second one in minimizing the local quadratic form. Theorem \ref{thm:idb_analysis} shows that the approximation error in the outer gradient grows linearly with those two errors. Assumption \ref{ass:ibd_analysis} ensures that the problem we are considering and its derivatives are well defined (\textit{i.} and \textit{ii.}) and that $\Li$ and $\Lo$ are regular enough (\textit{iii.} and \textit{iv.}).

We present a proof of Theorem~\ref{thm:idb_analysis} along with a discussion on how to transform the global convexity assumptions into local ones in Appendix~\ref{app:ift_extension_local}.

\subsection{Equilibrium propagation}

\label{sec:epb_methods}

Instead of differentiating through the implicit functions, we can resort to another mathematical result known as equilibrium propagation \citep{scellier_equilibrium_2017}, which reformulates the outer gradient in a way that is easier to estimate numerically. While equilibrium propagation was originally presented in the context of energy-based recurrent neural network learning, the result is far more general. As we discuss next, equilibrium propagation can be applied to solve general bilevel optimization problems.

\paragraph{Equilibrium propagation theorem.} The first step in equilibrium propagation consists in breaking up the hierarchy of losses and mixing $\Li$ and $\Lo$ in an augmented loss 
\begin{equation}
    \Lt(\phi, \theta, \beta) := \Li(\phi, \theta) + \beta \Lo(\phi, \theta).
\end{equation}
The nudging strength $\beta$ is a scalar that controls the strength of the mix; when it is equal to 0, we retrieve the inner learning problem. We denote by $\phisbt \in \argmin_\phi \Lt(\phi, \theta, \beta)$ the different minimizers of $\Lt$. We can now introduce the equilibrium propagation result.

\begin{theorem}[Equilibrium propagation \citep{scellier_equilibrium_2017, scellier_deep_2021}]
    \label{thm:equilibrium_propagation}
    Let $\Li$ and $\Lo$ be two twice continuously differentiable functions. Let $\bar{\phi}$ be a stationary point of $\Lt(\,\cdot\,, \bar{\theta}, \bar{\beta})$, i.e.,
    \begin{equation*}
        \pder{\phi}{\Lt}\left(\bar{\phi}, \bar{\theta}, \bar{\beta}\right) = 0,
    \end{equation*}
    such that $\partial_\phi^2 \Lt(\bar{\phi}, \bar{\theta}, \bar{\beta})$ is invertible. Then, there exists a neighborhood of $(\bar{\theta}, \bar{\beta})$ and a continuously differentiable function $(\theta, \beta) \mapsto \phisbt$ such that $\phi^*_{\bar{\theta}, \bar{\beta}} = \bar{\phi}$ and for every $(\theta, \beta)$ in this neighborhood we have
    \begin{equation*}
        \pder{\phi}{\Lt}\left (\phisbt, \theta, \beta \right ) = 0
    \end{equation*}
    and
    \begin{equation*}
        \label{eqn:ep_formula}
        \frac{\mathrm{d}}{\mathrm{d}\theta}\pder{\beta}{\Lt}\left(\phisbt, \theta, \beta \right) =  \frac{\mathrm{d}}{\mathrm{d}\beta}\frac{\partial \Lt}{\partial \theta}\left(\phisbt, \theta,\beta \right)^\top\!.
    \end{equation*}
\end{theorem}

\begin{proof}
    The existence part in the equilibrium propagation theorem directly follows from Theorem \ref{thm:ift_existence} using $f = \partial_\phi \Lt$. Obtaining the differentiation formula is not as complicated as it may appear at first glance. The first step consists of applying the symmetry of second-order derivatives result, also known as Schwarz's theorem:
    \begin{equation*}
        \der{\theta}{}\der{\beta}{} \Lt \left (\phisbt, \theta, \beta \right ) = \der{\beta}{}\der{\theta}{} \Lt \left (\phisbt, \theta, \beta \right )^\top\!.
    \end{equation*}
    We then apply the chain rule on both sides of the previous equation and use the equilibrium condition $\partial_\phi \mathcal{L}(\phisbt, \theta, \beta)=0$ to simplify the derivatives. For the left-hand side of the previous equation, it yields
    \begin{equation*}
        \begin{split}
            \der{\theta}{}\der{\beta}{} \Lt\left(\phisbt, \theta, \beta\right) & = \der{\theta}{} \left [ \pder{\beta}{\Lt}\left(\phisbt, \theta, \beta\right) + \pder{\phi}{\Lt}\left(\phisbt, \theta, \beta\right) \der{\beta}{\phisbt} \right ]\\
            & = \der{\theta}{} \pder{\beta}{\Lt}\left(\phisbt, \theta, \beta\right).
        \end{split}
    \end{equation*}
    The right-hand side can be simplified in the same way, which gives the desired formula.
\end{proof}

The equilibrium propagation result can be used to reformulate the outer gradient $\outg$ by remarking that $\partial_\beta \Lt = \Lo$ and $\left . \phisbt \right |_{\beta=0} = \phist$. We then have
\begin{equation}
    \label{eqn:outg_ep_form}
    \outg = \left . \der{\beta}{} \pder{\theta}{\Lt}\left(\phisbt, \theta, \beta\right) \right |_{\beta=0}.
\end{equation}

Theorem~\ref{thm:equilibrium_propagation} uses a stationary condition on the vector $\phi$ but more general versions of equilibrium propagation exist for stationary distributions or trajectories (see \citet{scellier_deep_2021} for more details). A very similar gradient estimate has also been derived when $\phi$ is a discrete quantity and the inner  and outer losses are expectations measured over a continuous distribution \citep{hazan_direct_2010, song_training_2016}.

There is a deep connection between the equilibrium propagation and implicit differentiation approaches: the quantity $\pi^*$ that we defined in Equation~\ref{eqn:def_pistar} is actually indirectly computed in equilibrium propagation, since $\pi^* = \left . \mathrm{d}_\beta \phisbt \right |_{\beta=0}$. The trajectories in the second phases of equilibrium propagation and implicit differentiation methods can also be shown to be closely related, when gradient descent is used in the second phase for the two methods \citep{scellier_equivalence_2019}.

\paragraph{Numerical estimation of $\outg$.}
The formula provided by the equilibrium propagation theorem might not appear useful at first. Closer inspection, however, reveals that it offers a new way of numerically estimating the outer gradient.  The outer gradient is a derivative of a scalar function with respect to a vector $\theta$, and is thus hard to estimate numerically, in particular when $\theta$ is high-dimensional. By contrast, the right-hand side of (\ref{eqn:outg_ep_form}) is the derivative of a vector-valued function with respect to a scalar, which can be readily estimated with finite difference techniques. The simplest finite difference estimator is:
\begin{equation}
    \label{eqn:eoutg_2ep}
    \eoutg^\top \coloneqq \frac{1}{\beta} \left ( \pder{\theta}{\Lt}(\phihb, \theta, \beta) - \pder{\theta}{\Lt}(\phihz, \theta, 0) \right )\!,
\end{equation}
where $\phihz$ and $\phihb$ are the approximated values of $\phiszt$ and $\phisbt$. This formula yields a two-phase algorithm that is detailed in Algorithm \ref{alg:comparison}. The approximation of the outer gradient can be refined by adding more points to the estimator, for instance by resorting to the central or forward finite difference estimators. The idea is to collect the value of $\partial_\theta \Lt$ at different values, e.g. $-\beta$ and $\beta$ for the central one with three points (as in \citet{laborieux_scaling_2021}) or $0, \beta, 2\beta, \ldots$ for the forward ones. We provide more details to the interested reader in Appendix~\ref{app:details_ep_estimators}.

\paragraph{Robustness to non-optimality \dag.} As for implicit differentation methods, it is possible to bound the error made by the two-point equilibrium propagation estimator (\ref{eqn:eoutg_2ep}). There are two sources of error: the approximation of the two solutions and the one rooted in the finite difference scheme. When $\beta$ gets smaller, the finite difference error gets smaller. On the other side, decreasing $\beta$ increases the sensitivity of the estimation to noise or inaccurate minimizations. Theorem \ref{thm:ep_analysis} quantifies it; we visualize the result on Figure~\ref{fig:viz_ep}.

\begin{restatable}{assumption}{ass_ep_analysis}
    \label{ass:ep_analysis}
    Assume that $\Li$ and $\Lo$ are three-times continuously differentiable. Additionally, suppose that there exists positive real numbers \linebreak $(B^\mathrm{in}, B^\mathrm{out}, L, \mu, \rho ,\sigma)$ such that $\Li$ and $\Lo$, as functions of $\phi$, verify the following properties:
    \begin{itemize}
        \item[i.] $\partial_\theta\Li$ is $B^\mathrm{in}$-Lipschitz and $\partial_\theta\Lo$ is $B^\mathrm{out}$-Lipschitz.
        \item[ii.] $\Li$ and $\Lo$ are $L$-smooth and $\mu$-strongly convex.
        \item[iii.] their Hessians are $\rho$-Lipschitz.
        \item[iv.] $\partial_\phi\partial_\theta\Li$ and $\partial_\phi \partial_\theta\Lo$ are $\sigma$-Lipschitz.
    \end{itemize}
\end{restatable}

\begin{restatable}[Error bound for equilibrium propagation \citep{zucchet_contrastive_2021}]{theorem}{ep_analysis}
    \label{thm:ep_analysis}
    Let $\beta > 0$ and $(\delta, \delta')$ be such that
    \begin{equation*}
        \snorm{\phihz-\phiszt} \leq \delta
    \end{equation*}
    and
    \begin{equation*}
        \snorm{\phihb-\phisbt} \leq \delta'.
    \end{equation*}
    Then, under Assumption~\ref{ass:ep_analysis}, there exists a $\theta$-dependent constant $C$ such that
    \begin{equation*}
        \snorm{\eoutg - \outg} \leq \frac{B^\mathrm{in}(\delta+\delta')}{\beta} + B^\mathrm{out}\delta' + C \frac{\beta}{1+\beta} \eqqcolon \mathcal{B}(\delta, \delta', \beta).
    \end{equation*}
\end{restatable}

\begin{figure}[ht]
    \centering
    \begin{subfigure}{200pt}
        \caption{}
        \includegraphics{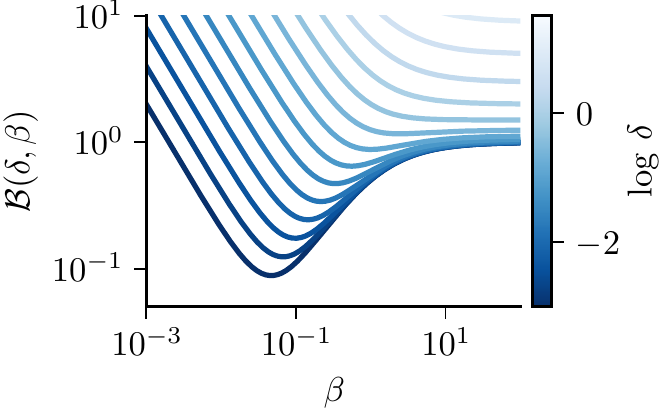}
    \end{subfigure}
    \begin{subfigure}{200pt}
        \caption{}
        \includegraphics{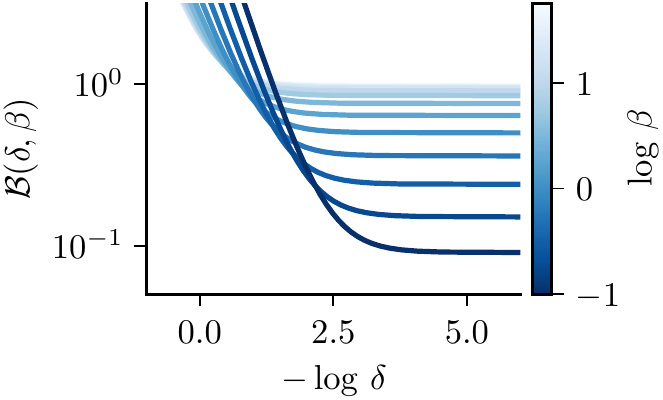}
    \end{subfigure}
    \caption{Visualization of the bound $\mathcal{B}$ on the gradient estimation error obtained in Theorem \ref{thm:ep_analysis}, as a function of $\beta$ (A) and as a function of $\delta=\delta'$ (B) (adapted from \citet{zucchet_contrastive_2021}).}
    \label{fig:viz_ep}
\end{figure}

As for implicit differentiation methods, we can obtain a more local version of Theorem~\ref{thm:ep_analysis} by replacing the strong convexity assumption of $\Li$ by a non-flat minimum assumption. The behavior of the estimator when the Hessian of $\Li$ at $\hat{\phi}$ is not positive definite is, however, quite different from the other kind of methods. Assuming that $\Li$ is bounded from below, the second phase ends up in a nearby basin of attraction in the worse case. The gradient estimator will then converge to some finite value, as opposed to implicit differentiation methods that will diverge.

\paragraph{Comparison with implicit differentiation methods \dag.} It is not yet possible to compare the bounds from Theorem~\ref{thm:idb_analysis} and Theorem \ref{thm:ep_analysis} as the bound for equilibrium propagation is still $\beta$-dependent. We can remove this dependency through the following corollary.

\begin{restatable}[Corollary of Theorem \ref{ass:ep_analysis} \citep{zucchet_contrastive_2021}]{corollary}{ep_analysis_opt}
    \label{cor:ep_analysis_opt}
    Under Assumption~\ref{ass:ep_analysis}, if we additionally suppose that for every $\beta > 0$ we approximate the two solutions with precision $\delta$ and $\delta'$ and if $(\delta+\delta')< C/B^\mathrm{in}$, the best achievable bound in Theorem~\ref{thm:ep_analysis} is smaller than
    \begin{equation*}
        B^\mathrm{out}\delta'+2\sqrt{CB^\mathrm{in}(\delta+\delta')}.
    \end{equation*}
\end{restatable}
The error made in the two-point equilibrium propagation estimator is therefore $O(\sqrt{\delta+\delta'})$, which implies that implicit differentiation methods are theoretically less sensitive to approximations in the two phases than equilibrium propagation. We compare in more details the two approaches in the next section.


\section{Comparison of the different approaches}

Having introduced implicit methods for bilevel optimization, the questions that come next are in which conditions they are useful, and which one to pick. The objective of this section is to help the give the reader insight into where the different methods shine, but not to give a definitive answer to such questions.

\subsection{Alternative methods} 

Backpropagation through time \citep{werbos_backpropagation_1990} can be used to compute gradients when the process used to estimate $\phist$ is a sequence of differentiable operations. In most settings, it is impossible to store the entire sequence of intermediate parameters produced by the algorithm in memory. The standard workaround to this problem is to run (truncate) the backward pass for a limited number of steps \citep{jaeger_tutorial_2002, shaban_truncated_2019} or to use a checkpointing strategy \citep{gruslys_memory-efficient_2016}. Whenever backpropagation or its truncated version is applicable, it is often a strong alternative to the implicit methods studied here; it is difficult to rule out a priori one class of methods over the other without experimenting with both.

There is, however, a number of clearly identifiable scenarios in which the methods discussed in this article may be preferable. Perhaps most importantly, it is not always possible to write the underlying optimization algorithm as a differentiable program. For example, an algorithmic solver can provably minimize a smooth loss function, but its inner process is not necessarily differentiable. In such cases, automatic differentiation is not an option, and implicit methods are in general the only gradient-based methods available. Furthermore, even when the learning algorithm is technically differentiable, it may generate chaotic sequences of parameters, which render gradients extremely noisy \citep{metz_understanding_2019}. In such situations, the methods studied here may lead to an implicit form of regularization of the learning process, by selecting outer parameters that are less prone to inducing chaos. More work is needed to investigate this hypothesis.

There is growing interest in physically-plausible learning algorithms, where optimization is performed by a physical system evolving in time \citep{millar_general_1951,kendall_training_2020, stern_supervised_2021, scellier_agnostic_2022}. It is generally impossible to implement backpropagation in such systems, as this would entail going back in (physical) time; even for reversible processes it is difficult to conceive backpropagation through time, since the computations performed in the forward and backward phases of this algorithm are not the same. Provided that the process which governs the time evolution of the parameters is differentiable, forward differentiation \citep[also known as real-time recurrent learning, cf.][]{williams_learning_1989} is the classic alternative to backpropagation which avoids going backwards in time. However, in its original form, this algorithm is typically infeasible to implement as well. First, its memory requirements scale with the dimension of $\theta$ multiplied by the dimension of $\phi$, which results in a huge memory cost. This is in fact a concern for most standard computer implementations as well. Second, the algorithm requires computing Jacobian-vector products, which may or may not be difficult to calculate in physical systems. Most of the concerns outlined above apply equally when looking at backpropagation or forward differentiation as biological learning algorithms. Some approximations have been developed to circumvent these limitations \citep{sutton_adapting_1992, tallec_unbiased_2018, bellec_solution_2020, marschall_unied_2020, menick_practical_2021}. 

\subsection{Comparison of the different implicit methods}

\begin{figure}[ht]
    \centering
    \begin{subfigure}{246pt}
        \caption{\normalfont \normalsize Implicit differentiation}
        \vspace{2pt}
        \includegraphics{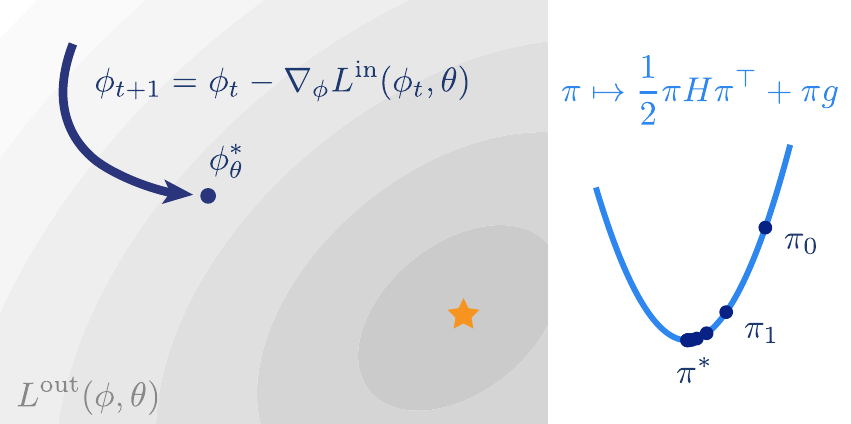}
    \end{subfigure}
    \hspace{4pt}
    \begin{subfigure}{157pt}
        \caption{\normalfont \normalsize Equilibrium propagation}
        \vspace{2pt}
        \includegraphics{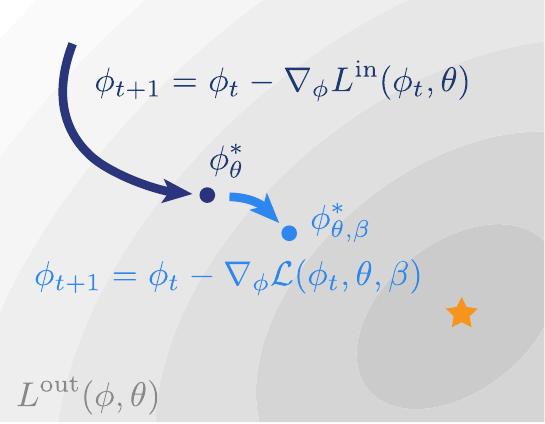}
    \end{subfigure}
    \caption{Visual comparison of implicit differentiation (A) and equilibrium propagation (B). The goal of bilevel optimization is to search for outer parameters $\theta$ such that the corresponding $\phist$ minimizes the outer loss $\Lo$ (loss magnitude is plotted in grey, the darker the smaller; the star denotes a minimum of $\Lo$). The dark blue curve represents the inner optimization dynamics, here chosen to be gradient descent, which is the same for the two approaches. The gradient calculation however differs. (A) An auxiliary quadratic form minimization problem is solved ($H$ here denotes the Hessian matrix $\partial_\phi^2 \Li(\phist, \theta)$ and $g$ the gradient $\partial_\phi \Lo(\phist, \theta)^\top$) and the output $\pi^*$ is then used to estimate the outer gradient. (B) The inner objective is nudged towards the outer objective through the augmented loss $\Lt$, which is then minimized (light blue curve). The outer gradient is then estimated by contrasting partial derivatives evaluated at $\phist$ and $\phisbt$.}
    \label{fig:comparison}
\end{figure}

Finally, we compare the methods presented in the previous sections. More concretely, we consider methods which use first-order (FO) approximations of the outer gradient\footnote{Note that, as we mentioned in Section~\ref{sec:idb_methods}, those methods can be seen as implicit differentiation methods with extremely short second phases}, equilibrium propagation (EP), recurrent backpropagation (RBP), and the conjugate gradient (CG) method. We determine use cases for the different methods based on three criteria: efficiency when all the theoretical assumptions are met, robustness to violation of the assumptions, and simplicity of the methods in terms of the computational elements involved. The result of the comparison is summarized in Table~\ref{tab:comparison_id} and a visual comparison of the algorithm is provided in Figure~\ref{fig:comparison}.

\begin{table}[]
    \renewcommand{\arraystretch}{1.3}
    \centering
    \begin{tabular}{@{}lrrr@{}}
        \toprule
         & Efficiency & Robustness & Simplicity \\
        \midrule
        First-order approximation & + & +++ & +++ \\
        Recurrent backpropagation & ++ & ++ & +\\
        Conjugate gradients & +++ & + & +\\
        Equilibrium propagation & ++ & ++ & ++\\
        \bottomrule
    \end{tabular}
    \caption{Summary of the comparison between methods rooted in implicit differentiation}
    \label{tab:comparison_id}
\end{table}

\paragraph{Efficiency under met assumptions.} We build our comparison upon the theoretical analysis presented in the last section (Theorem~\ref{thm:idb_analysis} and Corollary~\ref{cor:ep_analysis_opt}) and want to figure out which method produces the best estimate. It assumes that $\hat{\phi}$ is sufficiently close to $\phist$ so that all the methods are properly defined.

First-order methods here suffer as the estimation they provide cannot be refined to get closer to the outer gradient. For the remaining methods we have to consider two things: how sensitive is the gradient estimate to the approximations made in the two phases and which optimizers are used. Regarding the first point, implicit differentiation methods (CG and RBP) outclass equilibrium propagation. All methods minimize the same objective in the first phase so it is reasonable to consider that they use the same optimizer. For the second phase, RBP uses gradient descent, CG conjugate gradients, and EP whatever optimizer is best suited to the augmented objective. This implies that the algorithm with the best guarantees is CG as the optimizer it uses in the second phase is extremely efficient (hence more efficient than the one EP would use). The comparison between EP and RBP depends on the problem considered and requires empirical evidence.

\paragraph{Robustness to violated assumptions.} In the last paragraph, we assumed that we are sufficiently close to minimizing the inner loss so that all implicit methods are properly justified. We now look at how they behave when those conditions are not met. First-order methods here shine as they just perform a crude approximation and do not rely on those assumptions. In principle, both CG and RBP would have a diverging second phase if the Hessian is not positive definite but in practice, it seems that CG is much more unstable \citep{liao_reviving_2018, lorraine_optimizing_2020, grosse_lecture_2021}. EP does not have diverging second phase as long as $\beta \geq 0$ and the inner  and outer losses are bounded from below.

\paragraph{Simplicity of the computational elements.} The methods we compare here require different computational elements. While every method requires computing partial derivatives with respect to the outer parameters, approximate first-order methods stand out in their simplicity of implementation. In particular, these methods do not even require storing the result of the first phase. On the other hand, implicit differentiation methods are the most complex to implement as they involve calculating Hessian-vector products. In digital computers, automatic differentiation software offers efficient implementations of this operation (e.g., \citep{abadi_tensorflow_2016, paszke_pytorch_2019, bradbury_jax_2018}). However, implementing Hessian-vector products can be challenging in large-scale distributed systems, neuromorphic hardware, or more exotic analog physical systems. Arguably, it is also hard to conceive such an operation as being biologically plausible. Remarkably, equilibrium propagation only requires contrasting partial derivatives and, therefore, avoids computing such Hessian-vector-products. Recent developments \citep{scellier_agnostic_2022} on equilibrium propagation have shown that the outer gradient can still be estimated when the inner loss function underlying the (bio)physical system dynamics and its partial derivatives are unknown, as long as the parameters $\theta$ can be externally controlled. This considerably widens the scope of systems in which equilibrium propagation can be applied.

\paragraph{Which method to choose?} First-order methods tend to work best off-the-shelf, without extensive tuning, so they are a good choice if performance is not the most important criterion. When performance is important and inner optimization is easy enough so that it is possible to closely approximate a local minimum of the inner loss function, the conjugate gradient method is the best one. If it reveals to be too unstable, recurrent backpropagation might solve those instability issues. Finally, if computing Hessian-vector-products is not an option, but performance is still important, equilibrium propagation is worth being considered.


\section{Conclusion}

We have presented bilevel optimization in a broad machine learning context and discussed gradient-based methods to solve such problems. Framing learning as bilevel optimization generalizes the traditional cost-minimization view of learning to computations that are not necessarily explicitly described, and that therefore cannot be learned through gradient descent with backpropagated errors. The implicit methods we reviewed here, either rooted in implicit differentiation or equilibrium propagation, allow computing gradients for such problems using local information, and sometimes using only elementary operations. These properties may turn out to be of particular importance for the development of biological theories of learning, as well as for the development of next-generation learning machines.

\subsubsection*{Acknowledgements}
This research was supported by an Ambizione grant (PZ00P3\_186027) from the Swiss National Science Foundation and an ETH Research Grant (ETH-23 21-1) awarded to João Sacramento. We thank Benjamin Scellier, Johannes von Oswald and Simon Schug for their detailed comments on this manuscript.

\bibliographystyle{apalike}
\bibliography{refs}

\newpage

\appendix

\section{Some bilevel optimization learning problems}
\label{app:bo_formulations}

We here review different learning problems that fit in the bilevel optimization framework.

\subsection{Energy-based neural networks}


\paragraph{Explicit description of neural networks.}
Neural networks are usually described through the computations that they perform to process an input signal $x$. For a feedforward neural network it usually takes the following form:
\begin{equation}
    \phi^0=x, ~~\phi^{l+1} = \rho(W^{l}\phi^{l} + b^l)
\end{equation}
where $\phi^l$ corresponds to the activity of the neurons from the $l$-th layer, $\rho$ to a non-linear activation function, $W^l$ to the weights connecting layer $l$ to layer $l+1$ and $b^l$ to the biases of layer $l$. In the supervised learning framework, the activity $\phi^L$ at the very last layer is compared to a desired output $y$ through a cost function $C(\phi^L, y)$.
The backpropagation algorithm \citep{linnainmaa_taylor_1976, werbos_applications_1982, rumelhart_learning_1986} propagates the error measured at the last layer towards the first layers of the network to efficiently compute gradients and then learn the weights of the network. 

\paragraph{Energy-based description.} An alternative description of neural networks is to consider that the activity is an equilibrium of some energy function $E$. Going from an explicit to an implicit description is easy for the feedforward neural network\footnote{Similar manipulations can be done in general to obtain an implicit description of a system from an explicit one}: the energy function
\begin{equation}
    \label{eqn:pc_energy}
    E(\phi, \theta, x) := \frac{1}{2}\snorm{\phi^0-x}^2 + \frac{1}{2} \sum_{l=0}^{L-1} \snorm{\phi^{l+1} - \rho(W^l\phi^l+b^l)}^2
\end{equation}
has only one global minimizer which is the neuronal activity $\phist$, as computed through the feedforward processing described above. 

The energy-based formulation is more than a mathematical reformulation. For example, the energy \eqref{eqn:pc_energy} is derived from an approximate probabilistic approach in the predictive coding framework \citep{rao_predictive_1999, whittington_approximation_2017}. Other types of energy functions also exist, such as the Hopfield energy \citep{hopfield_neurons_1984, scellier_equilibrium_2017}, and encompasses computations that cannot be formulated explicitly. Although the term \textit{energy} has a physical meaning, physical networks can minimize other quantities than the physical energy, such as the co-content for electrical circuits \citep{millar_general_1951, kendall_training_2020}.

Under this paradigm, learning under supervision can be formulated as the following bilevel optimization:
\begin{equation}
    \begin{split}
        & \min_\theta ~ \mathbb{E}_{(x,y)}\left [ C(\phi_{\theta}^*,y) \right ]\\
        & \text{s.t.}~ \phi_{\theta}^* \in \argmin_\phi E(\phi, \theta, x).
    \end{split}
\end{equation}
Backpropagation is not generally applicable to compute gradients associated to this optimization problem. This is why we introduce implicit methods in this paper.

\subsection{Hyperparameter optimization and meta-learning}
\label{app:interactions_meta_learning}

We have briefly introduced bilevel optimization for hyperparameter optimization and meta-learning in Section~\ref{subsec:hp_ml}. Recall that in this context bilevel optimization generally takes the form
\begin{equation}
    \begin{split}
        & \min_\theta \mathbb{E}_{\tau} \left [ \Lo(\phi_{\tau,\theta}^*, \theta, \mathcal{D}_\tau^\text{val}) \right ]\\
        & \mathrm{s.t.}~ \phi_{\tau, \theta}^* \in \argmin_\phi \Li(\phi, \theta, \mathcal{D}_\tau^\text{train}),
    \end{split}
\end{equation}
where expectation is taken over multiple tasks $\tau$ for meta-learning and over a single one for hyperparameter optimization (the expectation then disappears).

The purpose of this section is to underline the diversity of interactions between inner  and outer parameters, which are also referred to as base and meta parameters in meta-learning. We have mentioned in Section~\ref{subsec:hp_ml} that the outer parameters can be the parameters of a quadratic regularization in the context of hyperparameter optimization \citep{goutte_adaptive_1998, bengio_gradient-based_2000} but the very same regularizer can be used in meta-learning \citep{rajeswaran_game_2020, zucchet_contrastive_2021}. In meta-learning, the center of the regularization is also meta-learned, providing a rough idea of which base parameter configuration works well on the task distribution. Another example is when the meta-parameters are the weights of a hypernetwork \citep{ha_hypernetworks_2017} that take the inner parameters as input to produce the weights of the network that processes incoming data \citep{lorraine_stochastic_2018, mackay_self-tuning_2019,rusu_meta-learning_2019, zhao_meta-learning_2020}. The task-specific modification can also be done at the neurons level \citep{zintgraf_fast_2019, mudrakarta_k_2019, zucchet_contrastive_2021}, while keeping the weights of the neural network shared across tasks. Alternatively, the task-shared outer parameters can be the weights of a neural network that acts as a feature extractor that will help a task-specific classifier or regressor parameterized by the base parameters to solve the task at hand \citep{raghu_rapid_2020, lee_meta-learning_2019, bertinetto_meta-learning_2019}.

\subsection{Generative adversarial networks} 

Generative adversarial networks \citep{goodfellow_generative_2014, metz_unrolled_2017} consist of a generative and a discriminative network that are learned in an adversarial fashion. The discriminator, parametrized by $\phi$ has to distinguish between samples generated by the generator and samples coming from the true data distribution $p(x)$. On the other side, the objective of the generative model $g_\theta$ is to generate samples that fool a perfect discriminator $D_{\phi^*_\theta}$. The corresponding optimization bilevel optimization problem is:
\begin{equation}
    \begin{split}
        & \min_\theta -\mathrm{E}_{z \sim \mathcal{N}(0, 1)} \left [ \log D_{\phi^*_\theta}(g_\theta(z)) \right ] \\
        & \mathrm{s.t.}~ \phi_\theta^* \in \argmin_\phi \mathrm{E}_{x \sim p(x)} \left [ \log D_\phi(x) \right ] + \mathrm{E}_{z\sim \mathcal{N}(0,1)} \left [ \log \left ( 1-D_\phi(g_\theta(z) \right ) \right ] \! .
    \end{split}
\end{equation}

\subsection{Actor-critic}

Some reinforcement learning problems can be formulated as bilevel optimization, such as actor-critics \citep{konda_actor-critic_2000}. The objective here is to learn an actor, which is a policy that tries to maximize the expected reward received while interacting with an environment. It receives help from a critic, an action-value function, which gives better feedback to the actor than the reward only. Following \citep{pfau_connecting_2016, yang_global_2019, zhou_online_2020, hong_two-timescale_2020}, training an actor-critic can be formulated as
\begin{equation}
    \begin{split}
        &\max_\theta ~ \mathbb{E}_{s \sim \rho, ~ a \sim \pi_\theta(\cdot|s)} \left [ Q_{\phi_\theta^*}(s, a) \right ]\\
        & \text{s.t.} ~ \phi^*_\theta \in \argmin \mathbb{E}_{s \sim \rho, ~ a \sim \pi_\theta(\cdot|s)} \left [ \left ( Q_\phi(s,a) - Q^{\pi_\theta}(s, a) \right )^2\right ] \!,
    \end{split}
\end{equation}
where $\rho$ is the initial state distribution, $\pi_\theta(\cdot|s)$ is the policy distribution parametrized by $\theta$ (the actor), $Q^{\pi_\theta}$ its corresponding Q-function and $Q_\phi$ the approximate Q network (the critic). A similar formulation exists in model-based reinforcement learning where the critic is replaced by a model which tries to predict the feature \citep{rajeswaran_game_2020}.

\subsection{Stackelberg games}

Interestingly the last two examples can be given a game-theoric interpretation through the notion of Stackelberg games. Stackelberg games \citep{von_stackelberg_market_1934} are a class of games where two players, a leader and a follower, play with hierarchical order. The leader $\theta$ has a strategic advantage: it plays first and knows what will be the perfect answer of the follower $\phi$. In the bilevel optimization framework, the follower's best response minimizes the inner loss and the leader optimizes the outer loss knowing the perfect answer of the follower $\phist$.

For generative adversarial networks, the generator is the leader and the discriminator the follower in the Stackelberg game terminology \citep{fiez_implicit_2020}. For actor-critic methods, the actor is the leader, as we ultimately want to get a good working policy, and the critic is the follower \citep{zheng_stackelberg_2021}.

\section{Theoretical analysis for implicit differentiation methods}
\label{app:idb_analysis}

\subsection{Proof of Theorem \ref{thm:idb_analysis}}

We here prove Theorem~\ref{thm:idb_analysis}. The proof is inspired from \citet{pedregosa_hyperparameter_2016}, which proves a very similar result under local assumptions, and \citet{rajeswaran_meta-learning_2019} which uses global assumptions and study the regularized inner loss we studied in Section~\ref{subsec:hp_ml}. The proof we here present uses the general formulation of the former with the stronger assumptions of the latter, in the goal of making the proof as insightful as possible to the reader.

Let us first rewrite the assumptions and the statement of the theorem.
\begin{repassumption}{ass:ibd_analysis}
    Suppose that there exists positive real numbers $(\mu, \rho, B, L, M)$ such that:
    \begin{enumerate}
        \item[i.] $\Li$ is twice continuously differentiable and $\Lo$ is continuously differentiable.
        \item[ii.] $\Li$ is $\mu$-strongly convex as a function of $\phi$.
        \item[iii.] The second-order derivatives (Hessian and cross derivatives) of $\Li$ are $\rho$-Lipschitz as functions of $\phi$.
        \item[iv.]  As functions of $\phi$, $\Lo$ is $B$-Lipschitz, $L$-smooth and $\spder{\theta}\Lo$ is $M$-Lipschitz,
    \end{enumerate}
\end{repassumption}
\begin{reptheorem}{thm:idb_analysis}
    Let $\phist$ be a minimizer of $\Li$ and $\hat{\phi}$ be its approximated value. Let $\hat{\pi}$ be an approximation of $\pi^*$ computed by one of the implicit differentiation methods. Let $\delta \in \, ]0, \frac{\mu}{2\rho}[$ be such that
    \begin{equation*}
        \snorm{\phist-\hat{\phi}} \leq \delta 
    \end{equation*}
    and $\delta' > 0$ such that
    \begin{equation*}
        \snorm{\hat{\pi} - \spder{\phi}{\Lo}(\hat{\phi}, \theta)\,\spder{\phi}{^2\Li}(\hat{\phi},\theta) ^{-1}} \leq \delta'.
    \end{equation*}
    Then, under Assumption~\ref{ass:ibd_analysis}, there exists a constant $C$ such that
    \begin{equation*}
        \snorm{\nabla_\theta - \widehat{\nabla}_\theta} \leq C(\delta+\delta').
    \end{equation*}
\end{reptheorem}

The main idea of the proof is to show that the outer gradient estimation error introduced by the implicit differentiation algorithm comes from two different sources: the fixed-point approximation error and the finite number of steps in the estimation of $\hat{\pi}$. Bounding the impact of the first source will be straight forward but the second one requires more work. This stems in the fact that implicit differentiation methods do not directly approximate $\pi^*$ in their second phase but only the proxy $\spder{\phi}{\Lo}(\hat{\phi}, \theta)\,\spder{\phi}{^2\Li}(\hat{\phi},\theta) ^{-1}$. We therefore need to quantify how far is the proxy from $\pi^*$. This can be done by remarking that the two are solutions of two similar linear systems. Lemma \ref{lem:perturbed_linearsys} is a result from perturbed linear systems theory that will allow us to upper bound the distance between the two.

\begin{lemma}[Theorem 7.2 \citep{higham_accuracy_2002}]
    \label{lem:perturbed_linearsys}
    Let $Ax=b$ and $A'x'=b'$ two linear systems with $\snorm{A-A'}\leq \varepsilon_A$ and $\snorm{b-b'}\leq \varepsilon_b$. If $\varepsilon_A \norm{A^{-1}} < 1$, then
    \begin{equation*}
        \snorm{x-x'} \leq \frac{\snorm{A^{-1}}}{1-\varepsilon_A\snorm{A^{-1}}}\left ( \varepsilon_b + \snorm{A^{-1}b}\varepsilon_A\right )\!.
    \end{equation*}
\end{lemma}

\begin{proof}
    Consider the quantity $A(x'-x)$. It is equal to
    \begin{equation*}
        \begin{split}
            A(x'-x) &= A'x' + (A-A')x' - Ax\\
            &= b' + (A-A')x' - b\\
            &= b'-b + (A-A')x+ (A-A')(x'-x).
        \end{split}
    \end{equation*}
    Then,
    \begin{equation*}
        x'-x = A^{-1}\left(b'-b + (A-A')A^{-1}b+ (A-A')(x'-x)\right)
    \end{equation*}
    so 
    \begin{equation*}
        \snorm{x'-x} \leq \snorm{A^{-1}} \left ( \varepsilon_b + \varepsilon_A\snorm{A^{-1}b} + \varepsilon_A \snorm{x'-x}\right )\!,
    \end{equation*}
    which yields the required result after subtracting $\varepsilon_A\snorm{A^{-1}}\snorm{x'-x}$ to both sides.
\end{proof}

With this result, we can now prove Theorem~\ref{thm:idb_analysis}.
\begin{proof}[Proof of Theorem \ref{thm:idb_analysis}]
    Recall that
    \begin{equation*}
        \nabla_\theta^\top = \pder{\theta}{\Lo}(\phist,\theta) - \pi^* \pder{\theta\partial\phi}{^2\Li}(\phist,\theta)
    \end{equation*}
    for
    \begin{equation*}
        \pi^* = \pder{\phi}{\Lo}(\phist, \theta) \left( \pder{\phi^2}{^2\Li}(\phist,\theta) \right)^{-1}
    \end{equation*}
    is estimated with
    \begin{equation*}
        \widehat{\nabla}_\theta^\top = \pder{\theta}{\Lo}(\hat{\phi},\theta) - \hat{\pi} \, \pder{\theta\partial\phi}{^2\Li}(\hat{\phi},\theta).
    \end{equation*}
    We introduce the shorthand $\chi^* \coloneqq \partial_\theta \partial_\phi \Li(\phist,\theta)$ and $\hat{\chi}$ its estimated counterpart. We then have
    \begin{align*}
        \snorm{\nabla_\theta - \widehat{\nabla}_\theta} &\leq \snorm{\partial_\theta \Lo(\phist, \theta)-\partial_\theta \Lo(\hat{\phi}, \theta)} + \snorm{\pi^*\chi^*-\hat{\pi}\hat{\chi}}\\ 
        &\leq \underbrace{\snorm{\partial_\theta \Lo(\phist, \theta)-\partial_\theta \Lo(\hat{\phi}, \theta)}}_{a)} + \underbrace{\snorm{\pi^*(\chi^*-\hat{\chi})}}_{b)} + \underbrace{\snorm{(\pi^*-\hat{\pi})\hat{\chi}}}_{c)}.
    \end{align*}
    We bound each term.
    \begin{itemize}
        \item[a)] From the Lipschitz continuity of $\partial_\theta \Lo$ comes 
            \begin{equation*}
                \snorm{\partial_\theta\Lo(\phist,\theta)- \partial_\theta\Lo(\hat{\phi}, \theta)} \leq M\snorm{\phist-\hat{\phi}} \leq M\delta.
            \end{equation*}
        \item[b)] The $\mu$-strong convexity of $\Li$ and the $B$-Lipschitz continuity of $\Lo$ implies that
            \begin{equation*}
                \snorm{\pi^*} \leq \snorm{\spder{\phi}{\Lo}(\phist, \theta)}\snorm{ \spder{\phi}{^2\Li}(\phist,\theta)^{-1}} \leq B \frac{1}{\mu}.
            \end{equation*} 
            Using the Lipschitz continuity of the cross derivatives of $\Li$ we have
            \begin{equation*}
                \snorm{\hat{\chi}-\chi^*} \leq \rho \delta
            \end{equation*}
            so
            \begin{equation*}
                \snorm{(\chi^*-\hat{\chi})\pi^*} \leq \rho\delta\snorm{\pi^*} \leq \rho\delta \frac{B}{\mu}.
            \end{equation*}
        \item[c)] Lipschitz continuity of $\partial_\theta\Li$ yields $\snorm{\partial_\phi\partial_\theta\Li(\hat{\phi}, \theta)} \leq M$. With the symmetry of the cross derivatives and the fact that the norm of a matrix equals the norm of its transpose, we have $\snorm{\hat{\chi}} \leq M$ and
        \begin{equation*}
            \snorm{\hat{\chi}(\pi^*-\hat{\pi})} \leq M\snorm{\pi^*-\hat{\pi}}
        \end{equation*}
        In the term $\snorm{\hat{\pi}-\pi^*}$, we still take in account the error made in the fixed point approximation. We can separate it with
        \begin{align*}
            \snorm{\hat{\pi}-\pi^*} \leq & ~ \snorm{\hat{\pi} - \spder{\phi}{\Lo}(\hat{\phi}, \theta)\,\spder{\phi}{^2\Li}(\hat{\phi},\theta)^{-1}}\\
            & + \snorm{\spder{\phi}{\Lo}(\hat{\phi}, \theta)\,\spder{\phi}{^2\Li}(\hat{\phi},\theta)^{-1}-\pi^*}\\
            \leq& ~ \delta'+ \snorm{\spder{\phi}{\Lo}(\hat{\phi}, \theta)\,\spder{\phi}{^2\Li}(\hat{\phi},\theta)^{-1}-\pi^*}.
        \end{align*}
        The second term now only depends on the fixed point approximation error and can be bounded using Lemma \ref{lem:perturbed_linearsys}. Due to the $\rho$-Hessian Lipschitz property of $\Li$, 
        \begin{equation*}
            \varepsilon_A \coloneqq \snorm{\spder{\phi}{^2\Li}(\phist, \theta)-\spder{\phi}{^2\Li}(\hat{\phi}, \theta)}\leq \rho\delta.
        \end{equation*}
        The use of the lemma is then justified by the upper bound assumption on $\delta$:
        \begin{equation*}
            \varepsilon_A \snorm{\partial_\phi^2\Li(\phist, \theta)^{-1}} \leq\rho\delta/\mu \leq 1/2 < 1.
        \end{equation*}
        The smoothness of $\Li$ implies
        \begin{equation*}
            \varepsilon_b \coloneqq \snorm{\partial_\phi \Lo(\phist, \theta) - \partial_\phi \Li(\hat{\phi}, \theta)} \leq L\delta.
        \end{equation*}
        We can now apply the lemma, which yields
        \begin{align*}
            \snorm{\spder{\phi}{\Lo}(\hat{\phi}, \theta)\,\spder{\phi}{^2\Li}(\hat{\phi},\theta)^{-1}-\pi^*} & \leq \frac{\mu^{-1}}{1-1/2}\left(L\delta + \snorm{\pi^*}\rho\delta\right)\\
            & \leq \frac{2\delta}{\mu}\left(L + \frac{B\rho}{\mu}\right)\!.
        \end{align*}
        We have therefore proved
        \begin{equation*}
            \snorm{\hat{\chi}(\hat{\pi}-\pi^*)} \leq M \left ( \delta' + \frac{2\delta}{\mu}\left(L + \frac{B\rho}{\mu}\right) \right )\!.
        \end{equation*}
    \end{itemize}
    Gathering the three bounds gives
    \begin{equation*}
        \snorm{\nabla_\theta - \widehat{\nabla}_\theta} \leq M\delta + \frac{B\rho}{\mu}\delta + M\left(\delta' + \frac{2\delta}{\mu}\left(L+\frac{B\rho}{\mu}\right)\right)\!.
    \end{equation*}
    Choosing
    \begin{equation*}
        C \coloneqq M + \frac{B\rho+2ML}{\mu} + \frac{2MB\rho}{\mu^2}
    \end{equation*}
    finishes the proof.
\end{proof}

\subsection{Extension to local assumptions}
\label{app:ift_extension_local}

In Theorem~\ref{thm:idb_analysis} we assumed the strong convexity of $\Li$ to get a bound on the outer gradient estimation error. We can get a more local version of it if we only assume that the Hessian at a minimum $\phi^*$ of $\Li$ is positive definite, i.e., that the minimum is not flat. The idea is to show that when the Hessian is continuous and it is positive definite at $\phist$, $\Li$ will be strongly convex in a neighborhood of $\phist$, which allows to go back to the assumptions of Theorem~\ref{thm:idb_analysis}. This is formalized in Fact~\ref{fact:local_strong_convexity}.
\begin{fact}
    \label{fact:local_strong_convexity}
    Let $\phi^*$ be a local minimum of $\Li$ such that $\partial_\phi^2\Li(\phi^*,\theta)$ is positive definite. Note $\mu$ its smallest (strictly positive) eigenvalue. If $\Li$ is $\rho$-Lipschitz Hessian, then $\Li$ is $\mu/2$-strongly convex on the ball of radius $\mu/2\rho$ centered on $\phi^*$.
\end{fact}

Interestingly, the Hessian of $\Li$ at $\hat{\phi}$ is not necessarily positive semi-definite when $\hat{\phi}$ outside the ball centered in $\hat{\phi}$ with radius $\mu/\rho$ (with the notations of Assumption~\ref{ass:ibd_analysis} and Fact~\ref{fact:local_strong_convexity}). In this case, the quadratic form (\ref{eqn:quadratic_form_ift}) is not bounded from below and procedures that try to minimize it will diverge.

\section{Equilibrium propagation estimators with multiple points}
\label{app:details_ep_estimators}

In Section~\ref{sec:epb_methods}, we have presented a way to estimate the outer gradient formula given by the equilibrium propagation theorem using 2 points. Recall that the equilibrium propagation allows to reformulate the outer gradient $\outg$ as
\begin{equation*}
    \outg = \left . \der{\beta}{}\pder{\theta}{\Lt}(\phisbt, \theta, \beta) \right |_{\beta=0}.
\end{equation*}
The simplest finite different estimator is the two points estimator that we have presented in Section~\ref{sec:epb_methods}:
\begin{equation*}
    \eoutg^\top = \frac{1}{\beta} \left ( \pder{\theta}{\Lt}(\phisbt, \theta, \beta) - \pder{\theta}{\Lt}(\phiszt, \theta, 0) \right )\!.
\end{equation*}
We now derive an estimator that uses several points to make a more accurate estimation of the derivative.

\paragraph{Forward finite differences.} The objective of this paragraph is to derive the $p$-forward finite difference learning rule that uses $p$ points to get a finer approximation of the outer gradient. Consider the values of the function $f:t \mapsto \partial_\theta\Lt(\phi_{\theta,t}^*, \theta, t)$ for $t \in \{0, \beta, \dots, (p-1)\beta\}$. We seek to find a linear combination of those measurements that approximates $\nabla_\theta=\evalat{\mathrm{d}_\beta\partial_\theta\Lt(\phisbt, \theta, \beta)}{\beta=0}=f'(0)$, i.e., find a vector $\alpha \in \mathbb{R}^p$ such that
\begin{equation}
    \label{eqn:finite_forward_diff_objective}
    \sum_{i=0}^{p-1} \alpha_if(i\beta) = \beta f'(0) + O(\beta^{p}).
\end{equation}
Taylor series approximation (around $\beta=0$) and an inversion of the summation indices yield
\begin{equation}
    \label{eqn:finite_forward_diff_estimation}
    \begin{split}
        \sum_{i=0}^{p-1} \alpha_if(i\beta) & = \sum_{i=0}^{p-1} \alpha_i \left ( \sum_{k=0}^{p-1}f^{(k)}(0)\frac{(i\beta)^k}{k!} + O(\beta^{p}) \right)\\
        & = \sum_{k=0}^{p-1} \sum_{i=0}^{p-1} \alpha_i f^{(k)}(0)\frac{(i\beta)^k}{k!} + O(\beta^{p})\\
        & = \sum_{k=0}^{p-1} f^{(k)}(0)\frac{\beta^k}{k!} \sum_{i=0}^{p-1} \alpha_i i^k+ O(\beta^{p}).
    \end{split}
\end{equation}
In (\ref{eqn:finite_forward_diff_objective}) and (\ref{eqn:finite_forward_diff_estimation}), we have two polynomials in $\beta$ that we want to be equal so all their coefficients have to be the same. We hence need to solve
\begin{equation}
    \label{eqn:coeff_system}
    \left ( i^k \right )_{i,k}\alpha = \left ( \begin{array}{c} 0\\1\\0\\ \vdots \end{array} \right )\!.
\end{equation}
where $\left ( i^k \right )_{i,k}$ is a $p\times p$ invertible Vandermonde matrix. The resolution of such a system can easily be done numerically. The values of $\alpha$ for small $p$ are:
\begin{equation*}
    \arraycolsep=10pt\def\arraystretch{1.2}
    \begin{array}{|c|ccccc|}
        \hline
        p & \alpha_0 & \alpha_1 & \alpha_2 & \alpha_3 & \alpha_4\\
        \hline
        2 & -1 & 1 & 0 & 0 & 0\\
        3 & -3/2 & 2 & -1/2 & 0 & 0\\
        4 & -11/6 & 3 & -3/2 & 1/3 & 0\\
        5 & -25/12 & 4 & -3 & 4/3 & -1/4\\
        \hline
    \end{array}
\end{equation*}
Note that the $2$-forward finite difference learning rule is the same as the finite difference one. In the following, we assume that $\alpha$ satisfies (\ref{eqn:coeff_system}). The resulting algorithm is presented in Algorithm \ref{alg:extended_EP}.
\begin{algorithm}[ht]
    \caption{$p$-forward finite difference learning rule}
    \label{alg:extended_EP}
    \KwResult{Estimation of $\nabla_\theta$}
    For every $i\in\{0, \dots, (p-1)\beta\}$, minimize $\phi\mapsto\Lt(\phi,\theta,i\beta)$, starting from the solution of last step, and note $\hat{\phi}_{i\beta}$ the result;\\
    Estimate $\nabla_\theta$, the derivative of $\theta \mapsto \Lo(\phist,\theta)$, using
    \begin{equation*}
        (\widehat{\nabla}_\theta^p)^\top = \frac{1}{\beta}\sum_{i=0}^{p-1}\alpha_i\pder{\theta}{\Lt}(\hat{\phi}_{i\beta}, \theta, i\beta).
    \end{equation*}
    Return $\widehat{\nabla}_\theta^p$;
\end{algorithm}

\paragraph{Why forward finite differences?} There exists different kind of finite difference estimators that use multiple points\footnote{See \url{https://en.wikipedia.org/wiki/Finite_difference}.}. We chose to present the forward difference ones above as they are the ones that only use estimates for positive $\beta$ values. We illustrate why this may be important on an example.

When using $p=3$ points, the bias reduction obtained with a forward estimate is similar to the obtained with the symmetric or central estimate
\begin{equation}
    \frac{1}{2\beta}\left ( \pder{\theta}{\Lt}(\phi^*_{\theta, \beta}, \theta, \beta) - \pder{\theta}{\Lt}(\phi^*_{\theta, -\beta}, \theta, -\beta) \right )
\end{equation}
that is used in \citet{laborieux_scaling_2021}. However, negative $\beta$ values can prove to be problematic. To illustrate that, consider $\Li$ and $\Lo$, two $\mu$-strongly convex and $L$-smooth functions (e.g. $\Li(\phi)=L \snorm{\phi}^2/2$ and $\Lo(\phi)=\mu \snorm{\phi}^2/2$). Then $\Lt$ (here equal to $(L+\beta\mu)\snorm{\phi}^2/2$) is not bounded from below when $\beta < -L/\mu$ so $\phi^*_{\theta,-\beta}$ does not exist anymore and the estimate diverges. When using negative $\beta$ values one therefore as to be careful that all the phases converge.

\paragraph{When are multiple points estimators worth it?} One can think that adding more points will always lead to a more precise estimation of the outer gradient. If the approximations $\hat{\phi}_{i\beta}$ of the minimizers $\phi^*_{i\beta, \theta}$ are perfect, it will always be the case as the bias is a $O(\beta^p)$ with $p$ the number of points. Using the terminology used in the proof of Theorem~\ref{thm:ep_analysis}, this means that the finite difference error decreases when the number of point increases. However, when we can only obtain approximate minimizers, adding more points aggregates the fixed-point approximation errors made in each phase thus potentially making the estimation error bigger. Whether more points would be useful or not is therefore a practical matter.

\end{document}